\documentclass[12pt]{jmlr}

\usepackage{amsmath,amssymb}
\usepackage{graphicx}
\usepackage{mathtools}
\usepackage{url}
\usepackage{xcolor}
\usepackage{soul}
\usepackage{cancel}
\usepackage{natbib}
\usepackage{bm}
\usepackage{thmtools, thm-restate}
\usepackage{framed}
\usepackage[graph,matrix,frame,arc,tips,all,poly]{xy}
\usepackage{tikz}
\usepackage{tikzscale}
\usepackage{forest}
\usepackage{tikz-cd}
\usetikzlibrary{shapes}

\colorlet{shadecolor}{yellow}

\jmlrproceedings{}{}

\newtheorem{mydef}{Definition}

\newtheorem{myth}{Theorem}

\newcommand{\myreal}{\ensuremath{\mathbb{R}}}

\newcommand{\mysob}{\ensuremath{W^{m,\infty}(I^n)}}
\newcommand{\mynorm}[1]{\ensuremath{||{#1}||_\infty}}

\title{Some Approximation Bounds for Deep Networks}
\author{\Name{Brendan McCane} \Email{mccane@cs.otago.ac.nz}
  \\ \Name{Lech Szymanski} \Email{lech@cs.otago.ac.nz}\\
  \addr Dept of Computer Science, University of Otago, Dunedin, New Zealand}

\begin{document}
\maketitle

\begin{abstract}
In this paper we introduce new bounds on the approximation of functions in deep networks and in doing so introduce some new deep network architectures for function approximation. These results give some theoretical insight into the success of autoencoders and ResNets.
\end{abstract}

\section{Introduction}

Deep networks have been shown to be more efficient than shallow for certain classes of problems: periodic functions \citep{szymanski2014deep}; radially symmetric functions \citep{eldan2016power}; and hierarchial compositional functions \citep{mhaskar2016deep}. Other work has shown that deep networks can efficiently represent low-dimensional manifolds \citep{basri2016efficient,shaham2016provable}, and \citet{telgarsky2016benefits} shows that there exist functions which cannot be efficiently represented with shallow networks.

All is not lost for shallow networks however. \citet{mhaskar1996optimal} gives bounds for the approximation of Sobolev functions using shallow networks and shows that these bounds are tight. There appear to be no similar bounds for deep networks on this class of functions. One might naturally ask if deep networks can approximate this class of functions with similar bounds, or if shallow networks are demonstrably superior in this case. If the former, then one may conclude that a deep network is never worse than a shallow counterpart and hence we should always favour deep networks. If the latter, then choosing a shallow network may often be a good choice. This paper goes some way to answering this question by establishing upper bounds of approximation on some specific deep network architectures.

\section{Definitions}

We follow many of the conventions of \citet{mhaskar2016deep}.

\begin{mydef}[$V_N$]
The set of all networks of a given kind with complexity N (the number of units in the network).
\end{mydef}

\begin{mydef}[$\mynorm{f}$] Norm of a function.
Let $I^n = [-1,1]^n$ be the unit cube in $n$ dimensions. Let $\mathbb{X}=C(I^n)$ be the space of all continuous functions on $I^n$ with:
\begin{equation}
\mynorm{f} = \max_{x \in I^n} | f(x) |.
\end{equation}
\end{mydef}

\begin{mydef}[Degree of approximation]
If $f$ is the unknown function to be approximated, then the distance between $f$ and an approximating network is:
\begin{equation}
\text{dist}(f, V_N) = \inf_{P \in V_N} \mynorm{f - P}.
\end{equation}
\end{mydef}

\begin{mydef}[\mysob] A Sobolev space.
Let $m \ge 1$ be an integer. Let \mysob be the set of all functions of $n$ variables with continuous partial derivatives of orders up to $m < \infty$ such that:
\begin{equation}
\mynorm{f} + \sum_{1 \le |k|_1 \le m} \mynorm{\mathcal{D}^k} \le 1,
\end{equation} 
where $\mathcal{D}^k$ denotes the partial derivative indicated by the multi-integer $k \ge 1$, and $|k|_1$ is the sum of the components of $k$.
\end{mydef}

\begin{mydef}[$\sigma$] The transfer function.
Let $\sigma: \myreal \mapsto \myreal$ be infinitely differentiable, and not a polynomial on any subinterval of $\myreal$. Further, we restrict ourselves to $\sigma \in \mysob$. Many common smooth transfer functions satisfy these conditions including the logistic function, $\tanh$, and softplus.
\end{mydef}

\begin{mydef}[$S_{N,n}$] The class of all shallow networks with $N$ units and $n$ inputs.
Let $S_{N,n}$ denote the class of shallow networks with $N$ units of the form:
\begin{equation}
x \mapsto \sum_{k=1}^N a_k \sigma(\langle w_k, x \rangle + b_k)
\end{equation}
where $w_k \in \myreal^n$, $a_k, b_k \in \myreal$. The number of trainable parameters in such a network is $(n+2) N$. Since $\sigma \in \mysob$, it should be obvious that $S_{N,n} \in \mysob$.
\end{mydef}

In all the deep networks we consider, one neuron in each layer after the first hidden layer is identified as the function approximation neuron. This allows us to progressively approximate the function of interest.

\begin{mydef}[$R_{N,n,l}$] The class of input residual networks with $N$ units per layer, 
where $l$ is the number of layers, and $g_{a,b,n} \in S_{a,n}$ is the $b^{th}$ layer of the network. The first hidden layer of the network receives just the input coordinates. Each subsequent layer has the input coordinates and all the previous layer as input. See Figure \ref{fig-input-resnet}.
\end{mydef}

\begin{mydef}[$D_{N_l,n,l}$] The class of cascade residual networks with $N_l$ units per layer,
where $l$ is the number of layers, and $g_{a,b,n} \in S_{a,n}$ is the $b^{th}$ layer of the network. The first hidden layer of the network receives just the input coordinates. Each subsequent layer has the input coordinates and the function approximation neuron of the previous layer as input. See Figure \ref{fig-cascade-resnet}.
\end{mydef}

\begin{mydef}[$L_{N_l^{\ge},n,l}$] The class of fully connected layer networks with $N_l\ge n$ units per layer. 
Each layer is fully connected to the next layer except for the function approximation neuron which is connected to the function approximation neuron and final output neuron only. In this case the number of neurons in each layer exceeds the input dimension. See Figure \ref{fig-fc-layernet}.
\end{mydef}

\begin{mydef}[$L_{N_l^{<},n,l}$] The class of fully connected layer networks with $N_l < n$ units per layer. 
See Figure \ref{fig-fc-layernet}.
\end{mydef}


\begin{figure}[ht]
\centering
\resizebox{\textwidth}{!}{
\includegraphics{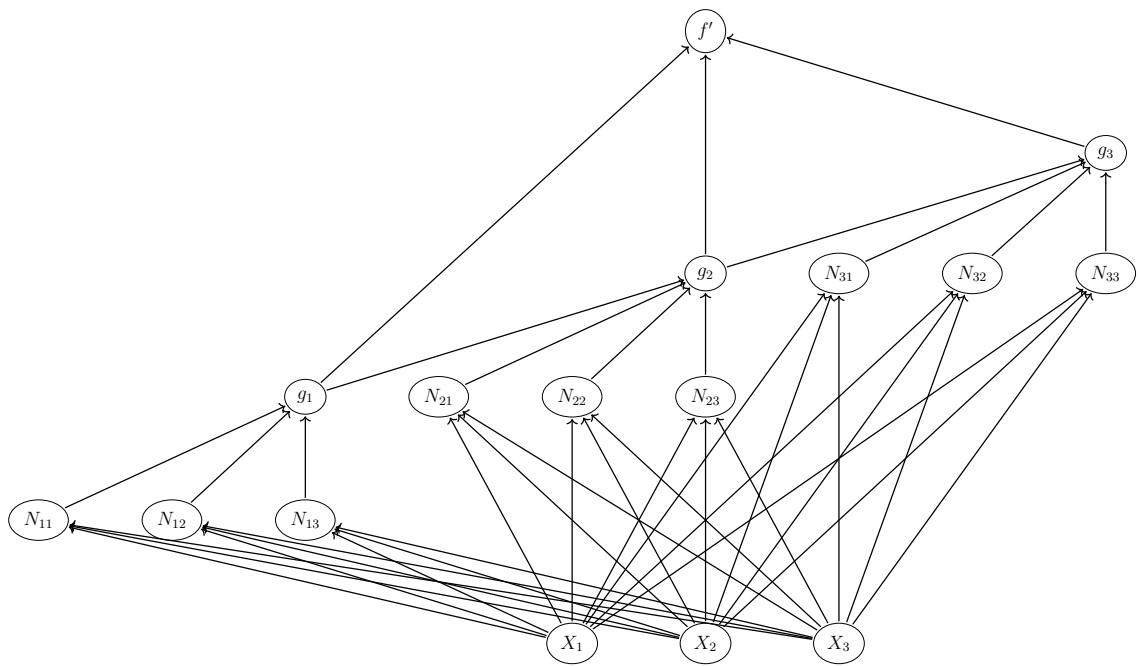}}
\caption{A CascadeResNet \label{fig-cascade-resnet}. $X_i$ are input nodes, $N_{ij}$ are nodes in layer $i$, and $g_i$ are the approximation outputs for layer $i$.}
\end{figure}

\begin{figure}[ht]
\centering
\resizebox{\textwidth}{!}{
\includegraphics{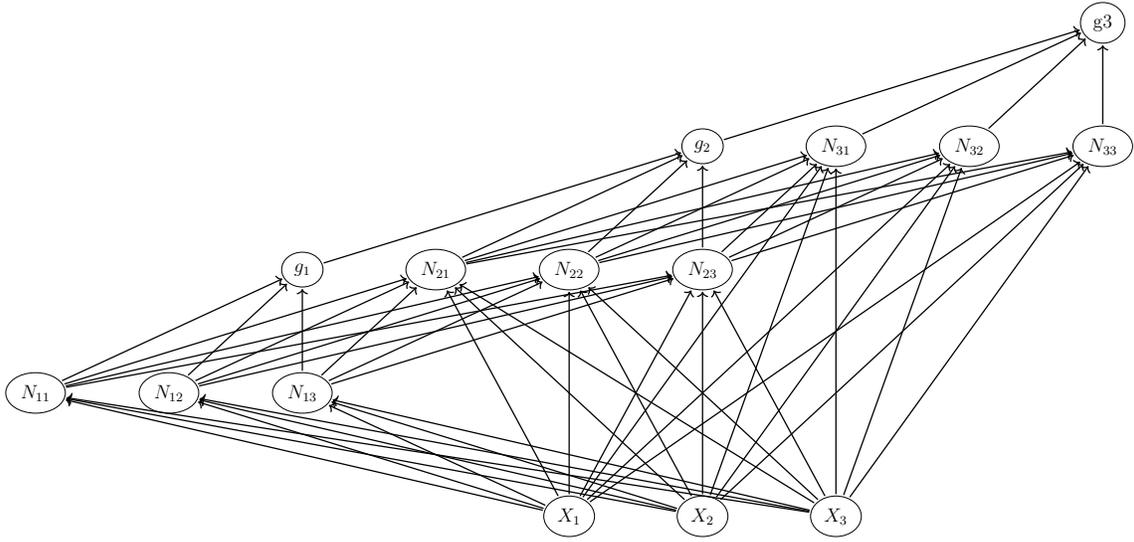}}
\caption{An InputResNet \label{fig-input-resnet}. $X_i$ are input nodes, $N_{ij}$ are nodes in layer $i$, and $g_i$ are the approximation outputs for layer $i$.}
\end{figure}

\begin{figure}[ht]
\centering
\resizebox{!}{0.5\textheight}{
\includegraphics{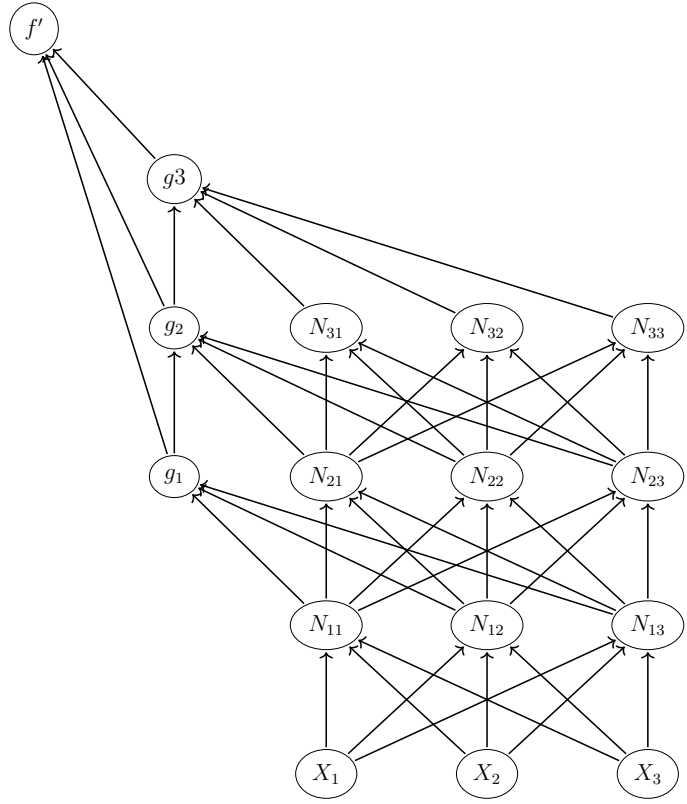}}
\caption{A FullyConnectedLayerNet \label{fig-fc-layernet}. $X_i$ are input nodes, $N_{ij}$ are nodes in layer $i$, and $g_i$ are the approximation outputs for layer $i$.}
\end{figure}


\section{Approximation Bounds}

We take as our starting point Theorem 2.1 of \citet{mhaskar1996optimal} also reported as Theorem 2.1(a) in \citet{mhaskar2016deep} and reproduce it here:

\begin{myth}[Theorem 2.1(a) of \citet{mhaskar2016deep}]
\label{th-shallow}
Let $\sigma: \myreal \mapsto \myreal$ be infinitely differentiable, and not a polynomial on any subinterval of $\myreal$. For $f \in \mysob$:
\begin{equation}
\text{dist}(f, S_{N,n}) \le c N^{-m/n},
\end{equation}
for some constant $c$.
\end{myth}

We use the results of Theorem \ref{th-shallow} to derive bounds for the deep network architectures defined above in the next three theorems.

\begin{myth}[Cascade Residual Network Approximation]
\label{thm-crna}
Let $\sigma: \myreal \mapsto \myreal$ be infinitely differentiable, and not a polynomial on any subinterval of $\myreal$. For $f \in \mysob$, and some constant $c$:
\begin{align}
\text{dist}(f, R_{N,n,l}) &\le c^l N_l^{-m(ln+1)/(n(n+1))}
\end{align}
\end{myth}
\begin{proof}
From Theorem \ref{th-shallow}, choose $g_1 \in S_{N,n}$, so that:
\begin{equation}
\mynorm{f-g_1} \le c_1 N_l^{-m/n}
\end{equation}
where $c_1$ is some constant. From the first layer, create a new function to approximate:
\begin{equation}
f_1 = \frac{1}{c_1 N_l^{-m/n}}(f-g_1)
\end{equation}
$f_1$ is clearly in $\mysob$ and note that $\mynorm{f_1} \le 1$, and therefore can be approximated with another single layer network ($g_2$), leading to the following error of approximation:
\begin{align}
\mynorm{f_1 - g_2} &\le c_2 N_l^{-m/(n+1)}\\
\mynorm{\frac{1}{c_1 N^{-m/n}}(f-g_1) - g_2}&\le c_2 N_l^{-m/(n+1)}\\
\frac{1}{c_1 N^{-m/n}} \mynorm{f - g_1 - c_1 N^{-m/n}g_2} &\le c_2N_l^{-m/(n+1)}\\
\mynorm{f - g_1 - c_1 N^{-m/n}g_2} &\le c_1 c_2 N_l^{-m(2n+1)/(n(n+1))}
\end{align}

Repeat the procedure by creating a new function to approximate:
\begin{equation}
f_2 = \frac{1}{c_1 c_2 N_l^{-m(2n+1)/(n(n+1))}} (f - g_1 - c_1 N_l^{-m/n)}g_2)
\end{equation}
Again $f_2 \in \mysob$ and $\mynorm{f_2} \le 1$. Approximate $f_2$ with a further layer ($g_3$):
\begin{align}
\mynorm{f_2 - g_3} &\le c_3 N_l^{-m/(n+1)} \\
\mynorm{\frac{1}{c_1 c_2 N_l^{-m(2n+1)/(n(n+1))}} (f - g_1 - c_1 N_l^{-m/n}g_2) - g_3}  &\le c_3 N_l^{-m/(n+1)}\\
\frac{1}{c_1 c_2 N_l^{-m(2n+1)/(n(n+1))}} \mynorm{f - g_1 - c_1 N_l^{-m/n}g_2 - c_1 c_2 N_l^{-m(2n+1)/(n(n+1))} g_3} &\le c_3 N_l^{-m/(n+1)} \\
\mynorm{f - g_1 - c_1 N_l^{-m/n}g_2 - c_1 c_2 N_l^{-m(2n+1)/(n(n+1))} g_3} &\le c_1 c_2 c_3 N^{-m(3n+1)/(n(n+1))}
\end{align}

A simple inductive argument completes the proof.
\end{proof}

Since $cN_l^{-m/n}<1$ (because $\mynorm{f}\le 1$, a constant function approximation would produce an error less than 1), it follows that the network will approximate the function exponentially fast in the number of layers.

\begin{myth}[Fully Connected Layer Network Approximation, $N_l\ge n$]
\label{thm-fclna-ge}
Let $\sigma: \myreal \mapsto \myreal$ be infinitely differentiable, and not a polynomial on any subinterval of $\myreal$. For $f \in \mysob$, and some constant $c$:
\begin{align}
\text{dist}(f, L_{N^{\ge},n,l}) &\le c^l N_l^{-m(ln+1)/(n(n+1))}
\end{align}
if each layer is an invertible map.
\end{myth}
\begin{proof}
Define $G_1: \myreal^n \mapsto \myreal^{N_l}$ as the (invertible) mapping for layer 1, $G_{i>1}: \myreal^{N_l} \mapsto \myreal^{N_l}$ as the (invertible) mapping for layer i excluding the function approximation neuron $g_{i-1}$ (see Figure \ref{fig-fc-layernet}). The input to node $g_1$ is $G_1(X)$; to $g_2$ is $g_1, (G_2 \circ G_1)(X)$; etc. Consider the input to $g_2$. Since $G_1$ is an invertible map, if necessary, we could construct the input to $g_2$ as $(G_2' \circ G_1^{-1} \circ G_1)(X) = G_2'(X)$. Since this is identical to the 
situation in Theorem \ref{thm-crna}, the same result applies.
\end{proof}

The next theorem deals with the case where $N_l<n$. However, in this case a continuous invertible mapping is not possible. Instead we project coordinates into a lower dimensional space using a Hilbert curve mapping to maintain locality (nearby points in the lower dimensional space are nearby in the original space). Theoretically, we could do this with no loss of information, but unfortunately, this requires an infinite recursion, and therefore any computational procedure will induce an error in the new coordinates. Nevertheless, this error can be made small with a fixed cost projection. 

A Hilbert curve can be defined by the centre coordinates of a hierarchically divided hypercube. See Figure \ref{fig-hilbert} for a 2D example. The curve itself, up to level $k$, can be constructed by recursively subdividing an initial square and creating line segments between appropriate centre points. The Hilbert curve itself is the limiting curve as $k$ goes to infinity and defines a continuous, but non-differentiable, onto mapping from $n$ dimensions to 1 dimension. There are several ways to generate Hilbert curve mappings both from $n$ dimensions to 1, and from 1 to $n$ dimensions. See \citet{lawder2000calculation} for one efficient method.
For level $k$, the maximum difference between a point in $[0,1]^n$ and a point on the curve is $\frac{\sqrt{n}}{2^{k+1}}$. 

\begin{figure}[ht]
\centering
\includegraphics{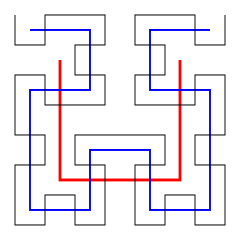}
\caption{The first 3 levels of a Hilbert curve. From \url{https://commons.wikimedia.org/wiki/File:Hilbert_curve_3.svg}, author Geoff Richards. \label{fig-hilbert}}
\end{figure}

\begin{myth}[Fully Connected Layer Network Approximation, $N_l < n$]
\label{thm-fclna-lt}
Let $\sigma: \myreal \mapsto \myreal$ be infinitely differentiable, and not a polynomial on any subinterval of $\myreal$. For $f$ Lipschitz with Lipschitz constant $L$, a Hilbert curve transform of level $k$, and some constant $c$:
\begin{align}
\text{dist}(f,L_{N^{<},n,l}) &\le \frac{L \sqrt{n-N_l+1}}{2^{k+1}} + c^l N_l^{-m(ln+1)/(n(n+1))}
\end{align}
\end{myth}
\begin{proof}
The proof is straightforward. For the coordinate projection from $n$ dimensions to $N_l$ dimensions, choose the first $n-N_l+1$ dimensions and apply a Hilbert curve transformation. This induces a coordinate error less than $\frac{\sqrt{n-N_l+1}}{2^{k+1}}$ and subsequently a function approximation error less than $L\frac{\sqrt{n-N_l+1}}{2^{k+1}}$. We then apply Theorem \ref{thm-fclna-ge} on the remaining $N_l$ coordinates along with the triangle inequality to prove the result.
\end{proof}

\section{Discussion}

The constants, $c$, in the theorems pose some difficulty since the errors are exponential in the number of layers ($c^l$). It appears to be possible to estimate the size of these constants \citep{dupont1978constructive,dupont1980polynomial}, however the process is not straightforward and we have not attempted to estimate them. Nevertheless these theoretical results provide hints that for more general functions deep networks are never much worse than shallow networks. Given previous results showing that deep networks can be much better than shallow for specific function classes, it follows that there is little to lose in always choosing deep architectures (modulo the difficulties in learning deep networks).

These theoretical results also point toward layer-wise learning algorithms that reduce error exponentially fast in a manner that is somewhat analagous to AdaBoost like algorithms. We are currently investigating the practical implications of such algorithms.

For layer-wise learning of fully connected networks, Theorem \ref{thm-fclna-ge}
 suggests that requiring invertible maps might be important. This might explain some of the success of autoencoders. Although Theorem \ref{thm-fclna-lt} suggests that non-invertible maps might be able to achieve similar results, via space-filling curve mappings, it remains to be seen if such a scheme would be practical.

More recently, ResNets \citep{he2016deep,szegedy2017inception} of various types have been shown to outperform non-ResNet architectures with the most common argument given for their success being that it is easier for the gradients to propagate back to the inputs during learning. Theorem 2 and 3 together suggest a second reason may be that it is also easier to approximate residual functions if layers are skipped as there is no requirement that the layer mapping be invertible.

\bibliography{../../../Biblio/papers}

\end{document}